\documentclass[11pt]{article}

\usepackage{authblk}
\usepackage[linesnumbered]{algorithm2e}
\usepackage{graphicx,amsmath,amsthm}
\usepackage{multirow}
\usepackage{subcaption}

\newtheorem{lemma}{Lemma}

\begin{document}

\title{Generation and Prediction of Difficult Model Counting Instances}

\author[1]{Guillaume Escamocher}
\author[1]{Barry O'Sullivan}

\affil[1]{Insight Centre for Data Analytics\\
School of Computer Science and Information Technology, University College Cork, Ireland}

\date{}

\maketitle

\begin{abstract}
We present a way to create small yet difficult model counting instances. Our generator is highly parameterizable: the number of variables of the instances it produces, as well as their number of clauses and the number of literals in each clause, can all be set to any value. Our instances have been tested on state of the art model counters, against other difficult model counting instances, in the Model Counting Competition. The smallest unsolved instances of the competition, both in terms of number of variables and number of clauses, were ours. We also observe a peak of difficulty when fixing the number of variables and varying the number of clauses, in both random instances and instances built by our generator. Using these results, we predict the parameter values for which the hardest to count instances will occur.
\end{abstract}

\section{Introduction}

In the field of Constraint Programming, hard to solve instances are highly sought after because they present an opportunity of improvement to solvers, by testing and challenging them. Valuable insights can be gained by running solvers on difficult instances, whether these benchmarks come from industrial applications~\cite{industrialinsights} or were specifically crafted to this aim~\cite{craftedinsights}. Various competitions provide regular opportunities for studying strngths and weaknesses of the latest constraint solvers. Competitions for satisfaction problems have been done for two decades~\cite{competitionsSAT,competitionsMZ}, while competitions for optimization problems have existed for almost as long~\cite{competitionsMaxSAT}. In contrast, the search for challenging model counting instances is a more recent endeavor, with the first edition of the Model Counting Competition only taking place in 2020~\cite{competition2020}. Because model counting has a wide range of applications, especially in the field of probabilistic reasoning~\cite{application2003,application2005,application2007,application2008,application2015}, studying difficult model counting instances is of great interest.

It is preferable for benchmark instances to exhibit certain properties in order to return useful feedback. First, generating them should be considerably faster than solving them. If it takes as long to build them than to solve them, then the number of available instances, and therefore the number of opportunities the solvers have to test themselves, will be limited. Second, the instances should be as diverse as possible, covering a wide array of sizes and structures. This is to ensure that the feedback obtained by running solvers on the instances is not restricted to a specific subset of the problem. Lastly, to better capture the challenging areas of model counting, it is desirable that the hardness of difficult instances comes from actually having to count the models, and not from the question of satisfiability. If a constraint instance has no solution and its proof of unsatisfiability is hard to establish for satisfaction solvers, then it represents a useful tool to improve satisfaction solvers, but is not as helpful for model counters because it does not address the distinct particularities of the model counting problem.

The instances produced by our generator fulfill all three above conditions. They are fast to build, in linear time in the number of variables times the number of literals. Their parameters (arity, number of variables, number of clauses) can be set to any value, even the arity of each individual clause can be adjusted independently of the others. Finally, while it is difficult to count their number of models, finding a single solution is easy. This means that whatever challenge the instances present is indeed related to the model counting problem.

Not all constraint instances are equally hard to solve, even when they belong to the same problem and contain the same number of variables. For random Conjunctive Normal Form instances, increasing the number of clauses while fixing the arity and the number of variables exposes the existence of two distinct phases, one where most instances are satisfiable, and the other where on the contrary most instances are unsatisfiable. The transition between the two phases is sharp~\cite{phasetransitionSAT} and corresponds to a peak in the difficulty of determining whether an instance has a solution~\cite{peakofdifficultySAT}, with the hardness of the problem decreasing as the number of clauses deviates from the transition position.

Basing the definition of the phase transition on satisfiability probability can of course only be done for problems where not all instances have a solution. However, other definitions have been adopted, revealing that the phase transition behavior occurs even for problems where all instances are satisfiable~\cite{phasetransitionsatisfiable}. For random model counting instances, phase transitions that depend on the probability for the number of solutions to be above a given threshold have been observed~\cite{phasetransitionsMC}. While the particular numbers picked by the authors in their experiments did not allow them to notice any peak of difficulty associated to these phase transitions, we will empirically show that these peaks do in fact exist for some other threshold values, at least for ternary instances.

The outline of our paper is as follows. In the next Section, we recall the definitions related to the general Satisfiability Problem and describe our generation algorithm in details. In Section~\ref{sec:competition}, we compare the performance of our instances against other benchmarks in the most recent Model Counting Competition, highlighting how our generator produced the smallest challenging instances of the benchmark pool. In Section~\ref{sec:empire}, we present the results of two sets of experiments. First, we show that when increasing the number of clauses, there emerges a peak in the difficulty of model counting instances, corresponding for ternary instances to a phase transition associated with a specific number of solutions threshold. Second, we show that we can use this data to reliably predict where in the constrainedness map will difficult model counting instances appear, without having to generate and solve them beforehand. Finally, we conclude in Section~\ref{sec:conclusion}.

\section{Generation}\label{sec:generation}

A \emph{Boolean} variable is a variable with a domain of size two. The values in the domain are typically called {\tt True} and {\tt False}. A \emph{literal} is either a Boolean variable or its negation. A \emph{clause} is a disjunction of literals, meaning that it is true if and only if at least one of its literals is true. The \emph{arity} of a clause is its number of literals. A \emph{Conjunctive Normal Form (CNF) instance} is a set of Boolean variables and a set of clauses on literals from these variables. A \emph{solution}, or \emph{model} to the instance is an assignment to all variables such that all clauses of the instance are true. If an instance admits a solution, then it is \emph{satisfiable}. Otherwise, it is \emph{unsatisfiable}.

The Satisfiability Problem, called SAT for short, consists in finding out whether a given CNF instance is satisfiable. This problem is NP-Complete, and was in fact the first problem proven to be so~\cite{sat}. If for some $k$ the arity of the clauses in a CNF instance is at most $k$, then the arity of the instance is $k$. We call this type of instances \emph{$k$-CNF}. When restricted to $k$-CNF instances, SAT is called $k$-SAT and it remains NP-Complete as long as $k\geq 3$~\cite{3sat}, although the problem is tractable if $k\leq 2$~\cite{2sat}.

\#SAT is the model counting version of SAT, that is the act of determining the number of solutions of a given CNF instance. \#SAT is \#P-Complete, even for binary instances~\cite{sharp2sat}.

Before presenting our generator proper, we describe the process of placing a solution $S$ in an instance $I$ in Algorithm~\ref{alg:placesolution}. The general idea is to check that each clause of $I$ satisfies $S$, and for any that does not we flip one of its literals $l$ to its negation $\overline{l}$. Only one literal in a clause needs to be true for the whole clause to be considered true, so flipping the polarity of any literal in each clause falsified by $S$ ensures that the resulting instance admits $S$ as a solution.

\begin{algorithm}
\caption{\label{alg:placesolution}PlaceSolution($I$,$S$)}
\KwData{An instance $I$ with $m$ clauses $C_1,C_2,\dots,C_m$ and a solution $S$.}
\KwResult{An instance that admits $S$ as a solution.}
\For{$i\leftarrow 1$ \KwTo $m$}
    {\If{no literal in $C_i$ is true in $S$}
        {$k\leftarrow$ arity of $C_i$\;
        $j\leftarrow$ random integer in $[1,k]$\;
        Flip polarity of the $j^\text{th}$ literal of $C_i$\;}}
\Return $I$\;
\end{algorithm}

We present our method to generate challenging model counting instances in Algorithm~\ref{alg:clustergenerator}. The first step is to create a CNF instance with the desired numbers $n$ and $m$ of variables and clauses, and the desired arities $k_1,k_2,\dots,k_m$. To do this, we build either a random CNF instance or a balanced one. The former simply consists in picking a random variable for each of the $\sum_{i=1}^m k_i$ literals in the instance, and negating it with probability $\frac{1}{2}$, with the only restriction being that no clause can be repeated. To build balanced instances on the other hand, we ensure that each variable appears the same amount of times in the instance (give or take one depending on the divisibility of $\sum_{i=1}^m k_i$ by $n$), with half of the occurrences being positive literals and the other half being negative literals.

It is usually easy to find a solution to a random CNF instance, at least for small sizes. In contrast, balanced instances are hard to solve~\cite{balancedsat}. For both types of instances, adding a cluster of solutions, as our generator does, makes (or keeps) the instance easy to solve for the satisfaction version of the problem: on a laptop with a 1.60GHz processor, the SAT solver Glucose 4.1~\cite{glucose} can solve all 480 instances that we submitted to the model counting competition in a combined time of less than two seconds. This demonstrates that the challenge posed by our instances comes entirely from the counting part, and not from determining satisfiability.

Once it has obtained an initial random or balanced instance, our algorithm then generates a random assignment to the $n$ variables and proceeds to place the $n+1$ solutions within a Hamming distance of 1 of this central model. By stuffing many solutions into a small sector of the assignment space, the goal is to make the distribution of the models as uneven as possible. Interestingly, this goes against a common heuristic for creating difficult satisfiability instances, which is to maximize balance~\cite{balancedsat,difficulttriangles}.

\begin{algorithm}
\caption{\label{alg:clustergenerator}GenerateClusterInstance($K$,$n$,$m$)}
\KwData{Two integers $n$ and $m$, and a set $K=\{k_1,k_2,\dots,k_m\}$ of $m$ integers.}
\KwResult{An instance with $n$ variables and $m$ clauses of arities $k_1,k_2,\dots,k_m$ that is challenging for model counters.}
$I\leftarrow$ instance with $n$ variables $v_1,v_2,\dots,v_n$ and $m$ clauses with $k_1,k_2,\dots,k_m$ literals\;
$S\leftarrow\emptyset$\;
\For{$i\leftarrow 1$ \KwTo $n$}
    {$r\leftarrow$ random integer in $[1,2]$\;
    \eIf{$r==1$}{$S\leftarrow S\cup\{(v_i={\tt True})\}$\;}
    {$S\leftarrow S\cup\{(v_i={\tt False})\}$\;}}
PlaceSolution($I$,$S$)\;
\For{$i\leftarrow 1$ \KwTo $n$}
    {Flip the assignment of $v_i$ in $S$\;
    PlaceSolution($I$,$S$)\;
    Flip back the assignment of $v_i$ in $S$ to its original value\;}
    \Return $I$\;
\end{algorithm}

The runtime of Algorithm~\ref{alg:clustergenerator} is linear in the total number of literals in the instance multiplied by the number of variables:

\begin{lemma}
Let $k$, $n$ and $m$ be three integers, and let $K=\{k_1,k_2,\dots,k_m\}$ be a set of $m$ integers such that $\max_{1\leq i\leq m}(k_i)=k$. Then GenerateClusterInstance($K$, $n$, $m$) runs in $O(k\times n\times m)$.
\end{lemma}

\begin{proof}
First we compute the complexity of placing a solution. In Algorithm~\ref{alg:placesolution}, the main loop (Line 1) iterates through the $m$ clauses. Checking whether a clause violates the input solution $S$ (Line 2) can be done by looping through all of its at most $k$ literals, while the operations in Lines 3-5 take $O(1)$ each. The final runtime is the product of the two loops, so $O(k\times m)$.

If the initial instance in Algorithm~\ref{alg:clustergenerator} is random, then the complexity of building it is the number of its literals, so $O(k\times m)$. If however we are generating a balanced instance, then for each literal we need to pick the variable with the fewest occurrences so far, which can be found in $O(n)$ by keeping track of the number of appearances of the variables. Therefore, Line 1 executes in $O(k\times n\times m)$. Getting a random model (Lines 2-10) can be done in $O(n)$, and we have shown that placing this central solution in the instance (Line 11) takes $O(k\times m)$. Placing $n$ more solutions (Lines 12-16) thus takes $O(k\times n\times m)$, so the total runtime of the algorithm is $O(k\times n\times m)+O(n)+O(k\times m)+O(k\times n\times m)$, which amounts to $O(k\times n\times m)$.
\end{proof}

The instances that we generated for the competition, as well as the instances that we tested in our experiments, have the same number of literals in all of their clauses. However, our algorithm is able to generate instances with clauses of heterogeneous sizes.

\section{Competition}\label{sec:competition}

The 2021 edition of the Model Counting Competition~\cite{competition2021} took place alongside the SAT conference of the same year~\cite{satconference}. It was comprised of four tracks, each with 100 instances. Tracks 2 and 3 were concerned with, respectively, weighted and projected model counting. While we did not target these tracks, the organizers selected some of our instances and assigned the weights and projected sets themselves. Since we did not have control over these additional parameters, we leave the results of these tracks to Appendix~\ref{sec:tracks23} and will only discuss tracks 1 and 4 in the current section.

We used random instances as the initial instance (Line 1 in Algorithm~\ref{alg:clustergenerator}) for half of our benchmarks, and balanced instances for the other half. For the former type, the number of variables $n$ was either 90 or 100. We had noticed in earlier tests that instances built upon balanced instances seemed to be slightly harder, so we generated slightly smaller instances of the second type, with values of 80 and 90 for $n$. For each instance, we picked a number of clauses $m$ such that $2.6n\leq m\leq 3n$. All of our instances have arity $k=3$. We submitted all the benchmarks that we generated, without filtering.

In order for a track 1 instance to be deemed solved by a solver, the solver was required to return in at most 60 minutes a model count within 1\% of the actual number of solutions of the instance. We present the rankings of the solvers that competed in this track in Table~\ref{tab:solverst1}. All results in this table, as well as all results from the other tables and figures in this section and in Appendix~\ref{sec:tracks23}, are the actual competition results. Some of the model counters were ran with two different configurations, we include both when this is the case. In addition of reporting the total number of instances solved for each model counter, we also distinguish the results between the instances that we submitted and instances submitted by others, as well as between small and large instances. Here we consider an instance to be small if either its number of variables $n$ or its number of clauses $m$ is no more than the corresponding value for the largest benchmark that we submitted, so either $n\leq 100$ or $m\leq 300$. In other words, small instances are the ones that have sizes similar to our instances.

\begin{table}[t]
\centering
\caption{Results by solver, track 1 instances.}
\label{tab:solverst1}
\begin{tabular}{|c|c|c c c|c|}
\hline
Solver & Ours & Not ours & Small, not ours & Large & Total\\
\hline
SharpSAT-TD & 19/20 & 59/80 & 6/6 & 53/74 & 78/100\\
nus-narasimha (cfg 1) & 5/20 & 55/80 & 4/6 & 51/74 & 60/100\\
nus-narasimha (cfg 2) & 5/20 & 54/80 & 6/6 & 48/74 & 59/100\\
d4 (cfg 2) & 0/20 & 51/80 & 6/6 & 45/74 & 51/100\\
d4 (cfg 1) & 0/20 & 50/80 & 6/6 & 44/74 & 50/100\\
gpmc & 20/20 & 18/80 & 6/6 & 12/74 & 38/100\\
DPMC & 0/20 & 34/80 & 4/6 & 30/74 & 34/100\\
MC2021\_swats (cfg 1) & 0/20 & 34/80 & 5/6 & 29/74 & 34/100\\
MC2021\_swats (cfg 2) & 0/20 & 31/80 & 5/6 & 26/74 & 31/100\\
c2d & 0/20 & 29/80 & 4/6 & 25/74 & 29/100\\
bob & 0/20 & 11/80 & 2/6 & 9/74 & 11/100\\
SUMC2 & 0/20 & 7/80 & 3/6 & 4/74 & 7/100\\
\hline
\end{tabular}
\end{table}

Out of nine distinct model counters, only three of them managed to solve at least one of our instances. Some solvers performed well on small instances when they were submitted by others but not when they came from our generator. For example, {\tt d4} was perfect in the former case but did not solve any of our benchmarks. The opposite, a model counter performing well on our instances but not on instances of similar size, did not occur.

Component caching solvers appear to perform the best against our instances. In particular, {\tt gpmc} was the only model counter to solve all of our benchmarks, despite ranking seventh out of nine on the other instances. On the other hand, the knowledge compilation model counter {\tt d4} could not solve any of our instances despite finishing third in the solver rankings. This could reveal a weakness of this type of solvers against instances produced by our generator.

We present the individual performance of track 1 instances in Figure~\ref{fig:resultst1}. The horizontal axis represents the size of the instance, either in terms of variables or in terms or clauses, and the vertical axis represents the number of solvers that could not output a correct model count for the instance within the one hour timeout. For model counters run with two configurations, each of the configurations is worth half a solver. Difficult instances are positioned near the top and easy ones near the bottom, while small instances can be found on the left and large ones on the right. We used different shapes and colors to visually separate our instances from other benchmarks. The number of instances in each set is indicated between parentheses.

\begin{figure}
\centering
\begin{subfigure}{\textwidth}
\centering
\includegraphics[width=.80\textwidth]{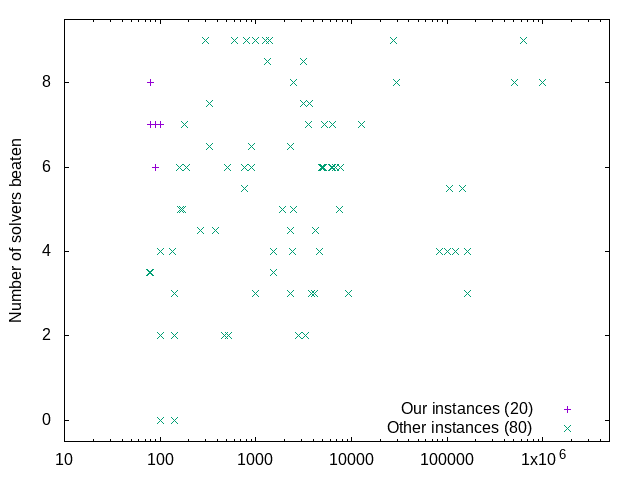}
\caption{Plotted against number of variables.}
\label{fig:resultst1_variables}
\end{subfigure}
\begin{subfigure}{\textwidth}
\centering
\includegraphics[width=.80\textwidth]{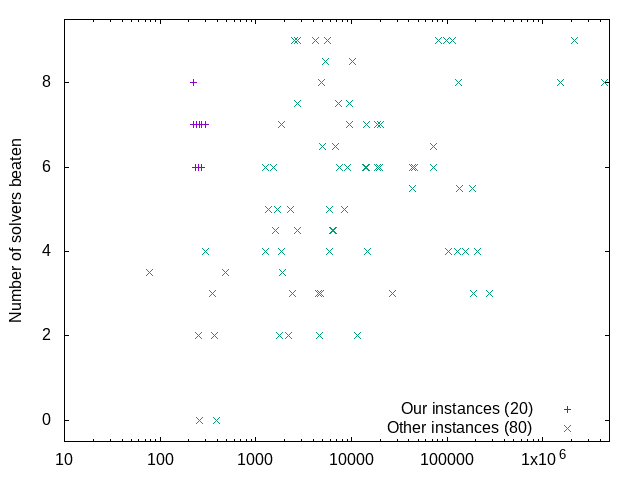}
\caption{Plotted against number of clauses.}
\label{fig:resultst1_clauses}
\end{subfigure}
\caption{Results for track 1 instances.}
\label{fig:resultst1}
\end{figure}

While none of our instances manages to beat all the solvers, it is clear from the plots that our benchmarks are the most challenging instances of their size, especially when size is defined in terms of number of clauses. When setting small size and difficulty as the two objectives, all but two instances at the Pareto frontier belong to our benchmark set.

In track 4, particularly difficult instances were selected and model counts returned by the solvers were allowed to deviate from the actual number of solutions by up to 80\% for an instance to be considered solved. The 60 minutes timeout remained the same. We present track 4 results by solver in Table~\ref{tab:solverst4} and by instance in Figure~\ref{fig:resultst4}.

\begin{table}[t]
\centering
\caption{Results by solver, track 4 instances.}
\label{tab:solverst4}
\begin{tabular}{|c|c|c c c|c|}
\hline
Solver & Ours & Not ours & Small, not ours & Large & Total\\
\hline
SharpSAT-TD & 7/13 & 61/87 & 16/16 & 45/71 & 68/100\\
Nus-narasimha (cfg 1) & 6/13 & 59/87 & 16/16 & 43/71 & 65/100\\
Nus-narasimha (cfg 2) & 6/13 & 58/87 & 16/16 & 42/71 & 64/100\\
d4 (cfg 2) & 0/13 & 53/87 & 16/16 & 37/71 & 53/100\\
d4 (cfg 1) & 0/13 & 52/87 & 16/16 & 36/71 & 52/100\\
c2d & 0/13 & 50/87 & 16/16 & 34/71 & 50/100\\
DPMC & 0/13 & 48/87 & 16/16 & 32/71 & 48/100\\
MC2021\_swats (cfg 1) & 0/13 & 40/87 & 16/16 & 24/71 & 40/100\\
MC2021\_swats (cfg 2) & 0/13 & 38/87 & 16/16 & 22/71 & 38/100\\
SUMC2 & 0/13 & 26/87 & 15/16 & 11/71 & 26/100\\
\hline
\end{tabular}
\end{table}

\begin{figure}
\centering
\begin{subfigure}{\textwidth}
\centering
\includegraphics[width=.80\textwidth]{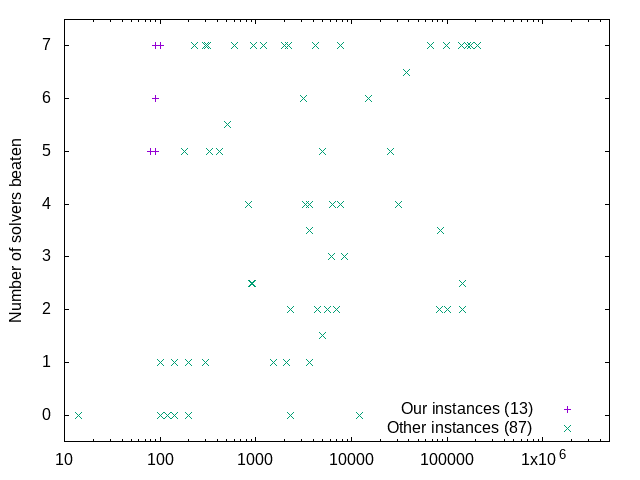}
\caption{Plotted against number of variables.}
\label{fig:resultst4_variables}
\end{subfigure}
\begin{subfigure}{\textwidth}
\centering
\includegraphics[width=.80\textwidth]{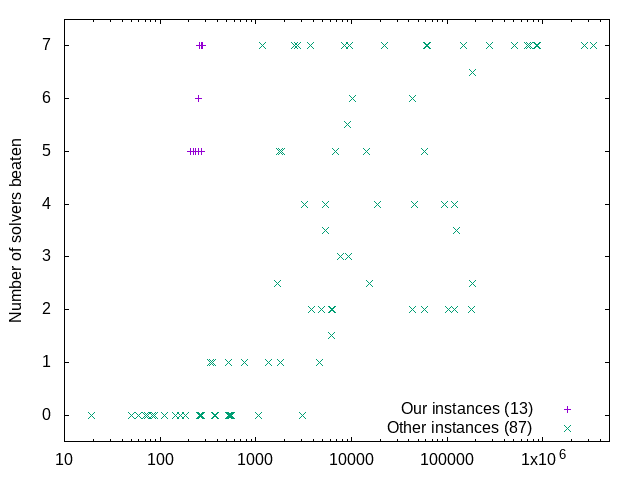}
\caption{Plotted against number of clauses.}
\label{fig:resultst4_clauses}
\end{subfigure}
\caption{Results for track 4 instances.}
\label{fig:resultst4}
\end{figure}

The performance of our instances was already decent in track 1, but it is even better in track 4. Only two model counters out of seven could solve at least one of our benchmarks. The majority of the solvers were perfect on small instances submitted by others, but could not solve any of our instances. In fact, every single model counter in the track performed better on \emph{large} instances than on the ones created by our generator. The smallest six instances that no model counter could solve were ours. None of our thirteen instances in the track was solved by more than two model counters, while all sixteen other small instances were solved by at least six of the seven model counters.

\section{Prediction}\label{sec:empire}

All experiments conducted in this section were run on a Dell PowerEdge R410 with an Intel Xeon E5620 processor. Three methods were used to generate the CNF instances tested:
\begin{itemize}
\item random: completely random $k$-CNF instances with $n$ variables, where the set of $k$ variables for each clause with $k$ literals is picked uniformly at random among the $\binom{n}{k}$ possibilities, and the polarity of each literal is determined by an independent coin toss.
\item random+solution: random CNF instances to which exactly one solution is placed in the way described in Algorithm~\ref{alg:placesolution}.
\item random+cluster: CNF instances obtained from our own Algorithm~\ref{alg:clustergenerator} generator, with a random instance as the Line 1 base.
\end{itemize}

Random instances were chosen to have a control dataset with instances that we expected to be easy, or at least not as challenging as our random+cluster instances. We added random+solution instances to the tests to have a kind of instances that is a compromise between the other two types. Random+solution instances are not as structured as random+cluster instances, but unlike completely random instances they cannot contain a small subset that is trivially unsatisfiable. However, as we will observe in Figures~\ref{fig:solversK3} and~\ref{fig:nbsolutions}, in practice random+solution instances behave in almost exactly the same way as standard random instances.

Whether a CNF instance is satisfiable or not does not depend on the tool used to solve it. Therefore the location of the phase transition between satisfiability and unsatisfiability is solver independent, as is the location of the peak of difficulty for the satisfaction version of the problem. This peak might be taller for some solvers than for others, but it will always happen at the same \#clauses/\#variables ratio.

As mentioned before, we cannot use (un-)satisfiability to characterize the location of the peak of difficulty for \#SAT. If a $k$-CNF instance with $n$ variables has no clauses, then its number of models is trivially $2^n$. If on the other hand the instance contains all possible $\binom{n}{k}\times 2^k$ different clauses, then its number of models is trivially 0. It is reasonable to believe that the peak of difficulty for model counting will indeed happen, at a constrainedness ratio that is somewhere in between these two extremes. It is however not obvious why it should happen at the same location for all solvers. For this reason we elected to use three different model counters to carry out the experiments:
\begin{itemize}
\item Cachet\cite{cachet}: a well established model counter with component caching and clause learning.
\item {\tt gpmc}\cite{gpmc}: this component caching model counter performed relatively poorly during the competition on other benchmarks, but was the best one on our instances, being the only model counter able to solve them all in Track 1. We chose it for this reason.
\item {\tt d4}\cite{d4}: the reason why we chose this knowledge compilation model counter is the opposite of the one why we picked {\tt gpmc}: {\tt d4} performed well overall in the competition, finishing third out of nine in Track 1, but could not solve a single one of our benchmarks.
\end{itemize}

We compare these three model counters on 3-CNF instances with 80 variables in Figure~\ref{fig:solversK3}. The ratio between the number of clauses and the number of variables ranges from 1 to 5, with a step of 0.1. For each number of clauses $m$ in this set, we generated 50 instances of each type. For all three generation methods, we can see that {\tt gpmc} performs the best, that {\tt d4} has the highest runtimes, and that Cachet consistently finishes between the other two. While the three model counters are not equally fast in solving the instances, their respective runtimes increase and decrease together, and they all peak at the same place. This would seem to indicate that the peak of difficulty for counting the number of models of 3-CNF instances is indeed associated with some phase transition that is independent of the model counter used.

\begin{figure}
\centering
\begin{subfigure}{\textwidth}
\centering
\includegraphics[width=.55\textwidth]{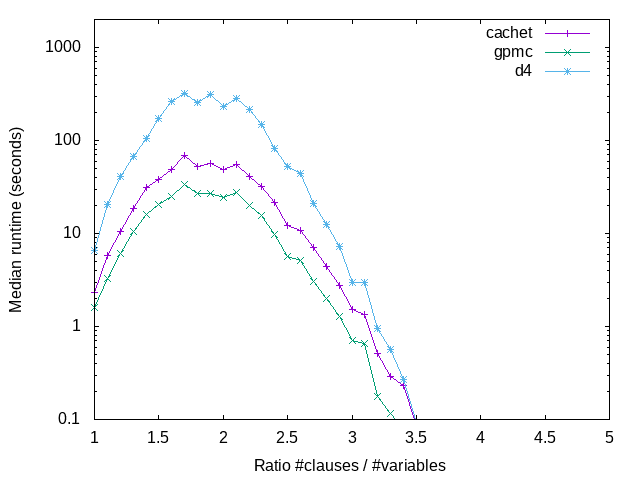}
\caption{Random instances.}
\label{fig:solversK3G1}
\end{subfigure}
\begin{subfigure}{\textwidth}
\centering
\includegraphics[width=.55\textwidth]{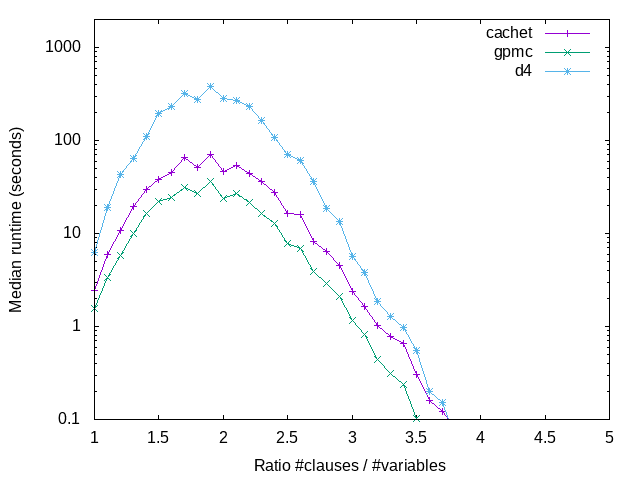}
\caption{Random+solution instances.}
\label{fig:solversK3G2}
\end{subfigure}
\begin{subfigure}{\textwidth}
\centering
\includegraphics[width=.55\textwidth]{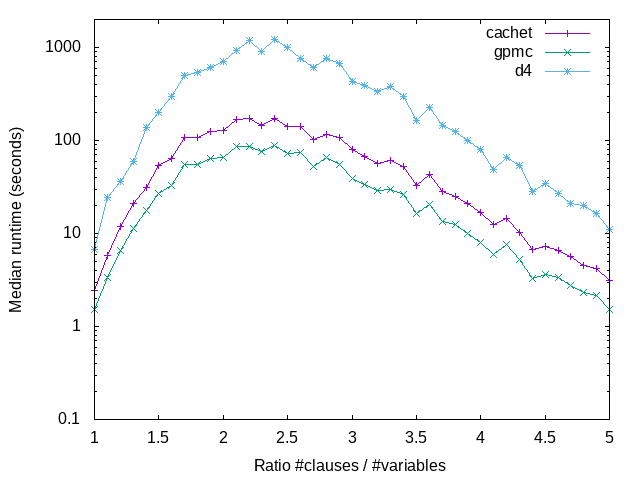}
\caption{Random+cluster instances.}
\label{fig:solversK3G3}
\end{subfigure}
\caption{Model counter comparison for 3-CNF instances with 80 variables.}
\label{fig:solversK3}
\end{figure}

We present in Figure~\ref{fig:solversG1} how the comparison between model counters evolves when the arity of the instances is increased beyond 3. For 4-CNF instances, the changes in difficulty experienced by the model counters do not happen at the exact same rate, unlike what we witnessed for ternary instances. For example, the runtimes experienced by Cachet do not decrease as fast in the highly constrained region as the runtimes for {\tt d4}. Consequently, the relative order of the model counters does not stay the same when increasing the constrainedness ratio. For 5-CNF instances, the differences are even more drastic, with the three peaks of difficulties occurring at clearly distinct places of the constrainedness map. This shows that for arities strictly higher than 3, there is no model counter independent phase transition that can characterize where the peaks of difficulty will happen because, while these peaks do exist, they are found at different places depending on the model counter used.

\begin{figure}
\centering
\begin{subfigure}{\textwidth}
\centering
\includegraphics[width=.55\textwidth]{Plots/solvers_G1_K3.png}
\caption{$k=3$, $n=80$.}
\label{fig:solversG1K3}
\end{subfigure}
\begin{subfigure}{\textwidth}
\centering
\includegraphics[width=.55\textwidth]{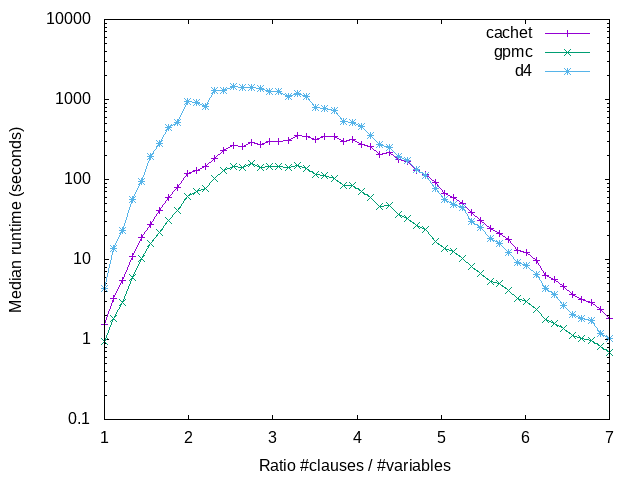}
\caption{$k=4$, $n=55$.}
\label{fig:solversG1K4}
\end{subfigure}
\begin{subfigure}{\textwidth}
\centering
\includegraphics[width=.55\textwidth]{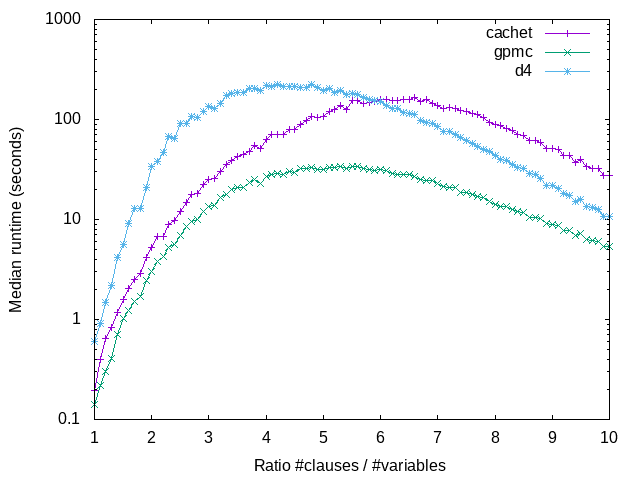}
\caption{$k=5$, $n=40$.}
\label{fig:solversG1K5}
\end{subfigure}
\caption{Model counter comparison for random instances.}
\label{fig:solversG1}
\end{figure}

Figure~\ref{fig:solversG1} contained the model counter results for random instances. The relative performance of the model counters exhibits a similar behavior for the other two types of instances, as can be seen in Figures~\ref{fig:solversG2} and~\ref{fig:solversG3} in Appendix~\ref{sec:solversG23}.

An interesting observation that can be made from Figure~\ref{fig:solversK3} (as well as from Figures~\ref{fig:solversG1} and~\ref{fig:solversG2}), is that while our own random+cluster instances are more difficult to solve than random and random+solution instances, the latter two give nearly identical runtimes, no matter which model counter is used. In addition of that, we show in Figure~\ref{fig:nbsolutions} that the actual number of solutions of a random instance and of a random+solution instance with the same arity, number of variables and number of clauses are the same. The only exception is for 3-CNF instances with a constrainedness ratio over 3.5, where placing a solution prevents the number of models to fall as quickly as it would in a completely random instance. Note that we know from Figure~\ref{fig:solversK3} that these over-constrainedness instances are trivial to solve anyway. Since random and random+solution instances behave the same in all relevant regions of the constrainedness map, we will only study random (and random+cluster) instances in the rest of this section.

\begin{figure}
\centering
\includegraphics[width=.75\textwidth]{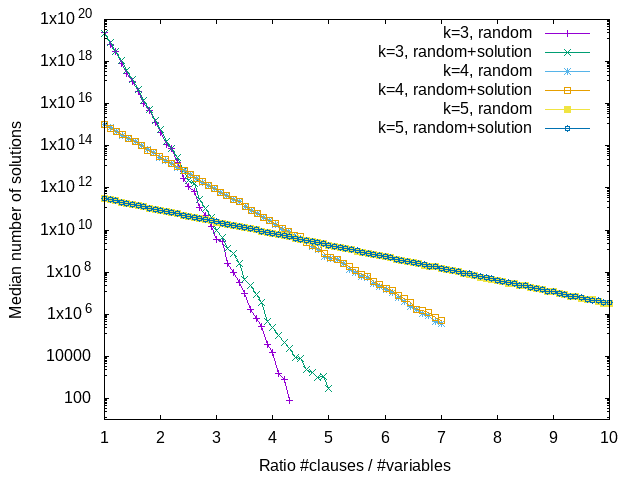}
\caption{Number of solutions for 3-CNF instances with 80 variables and constrainedness ratio from 1 to 5, 4-CNF instances with 55 variables and constrainedness ratio from 1 to 7, and 5-CNF instances with 40 variables and constrainedness ratio from 1 to 10.}
\label{fig:nbsolutions}
\end{figure}

We have shown that the location of the peak of the difficulty depends on the model counter used when the arity $k$ is strictly greater than 3, so we will restrict our prediction efforts on ternary CNF instances. We have also established that for $k=3$ the three model counters we ran all have their peaks of difficulty near the same constrainedness value. Therefore, from now on we will only present the results of one model counter for each number of variables $n$. If $n>80$, then the instances become too challenging to be solved in a reasonable time by Cachet or {\tt d4}, so we use {\tt gpmc}. We will keep Cachet for $n=80$, and will use {\tt d4} for instances with fewer variables.

We present in Figure~\ref{fig:peaks} our experimental results for random and random+cluster 3-CNF instances. The numbers of variables studied are 70, 80, 90 and 100. As in all other figures in this section, for each generation method 50 instances were tested for each combination of the number of variables $n$ and the number of clauses $m$. The plots show that as the size of the instances varies, the peak of difficulty for each type of instances remains in the same location. For random instances, it can always be found just before the ratio $\frac{m}{n}$ reaches 2. For random+cluster instances, it occurs a little later, when the ratio is between 2 and 2.5.

\begin{figure}
\centering
\begin{subfigure}{\textwidth}
\centering
\includegraphics[width=.80\textwidth]{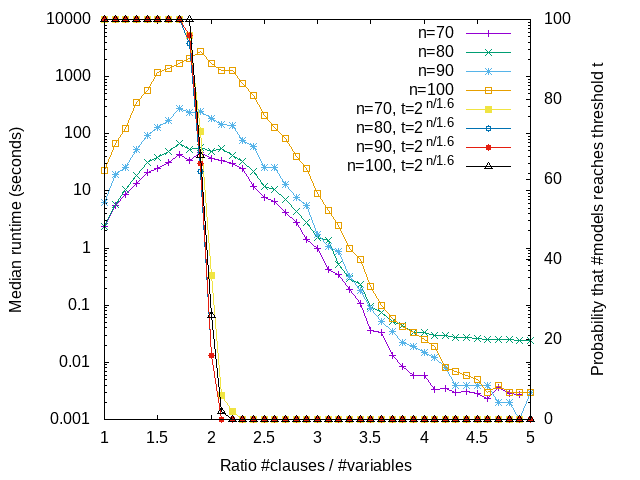}
\caption{Random instances.}
\label{fig:peaksG1}
\end{subfigure}
\begin{subfigure}{\textwidth}
\centering
\includegraphics[width=.80\textwidth]{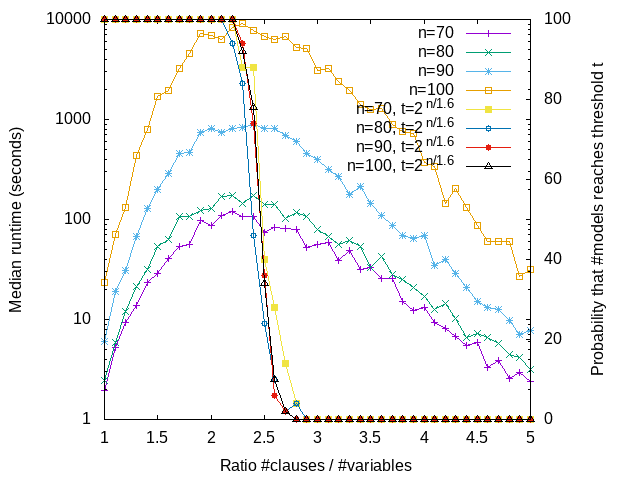}
\caption{Random+cluster instances.}
\label{fig:peaksG3}
\end{subfigure}
\caption{Positions of the peaks of difficulty and associated phase transitions for 3-CNF instances.}
\label{fig:peaks}
\end{figure}

To find a phase transition associated with the location of the peaks of difficulty, we look at the probability for the number of solutions to reach a given threshold. When this threshold is 2 to the power of a fraction of $n$, then it is known that phase transitions exist~\cite{phasetransitionsMC}. While no peak of difficulty was observed by the authors, probably because the threshold they studied was either too high or too low, we show in Figure~\ref{fig:peaks} that if the targeted number of solutions is $2^\frac{n}{1.6}$, then there is a sharp transition between a phase where all instances have more solutions than this threshold, and a phase where none of the instances does, and that this transition corresponds exactly to where the hardest instances are. Moreover, while the peaks of difficulty for random and random+cluster instances occur at different values of the constrainedness ratio $\frac{m}{n}$ (2 and 2.5 respectively), in terms of number of solutions they are found at the same threshold value ($2^\frac{n}{1.6}$). This suggests that a phase transition based on this particular number of solutions is a more general indicator of the location of the peak of difficulty for model counting than the constrainedness ratio.

We now create difficult ternary model counting instances for a higher number of variables $n$, namely 120. Generating instances for many constrainedness values, solving them all and keeping the most challenging ones would take too long to be feasible, so instead we use our results to immediately select the number of clauses $m$ most likely to yield hard instances. We observed a peak of difficulty around $m=1.9n$ for random instances and around $m=2.3n$ for random+cluster instances, so we build random instances with $1.9\cdot 120=228$ clauses and random+cluster instances with $2.3\cdot 120=276$ clauses. We also set a timeout of 6 hours for random instances, and of 24 hours for the more difficult random+cluster instances. As we show in Table~\ref{tab:timeouts}, timeout was reached on every single of the 100 3-CNF instances with 120 variables that we generated. This confirms that we can get hard ternary model counting instances for any number of variables by simply picking the same ratio $\frac{m}{n}$ for a given generation method.

\begin{table}[t]
\centering
\caption{Predicting where we can find challenging model counting instances.}
\label{tab:timeouts}
\begin{tabular}{|l|c|c|c|c|l|}
\hline
Generation method & $k$ & $n$ & $m$ & Number of instances & Minimum runtime\\
\hline
Random & 3 & 120 & 228 & 50 & timeout (6 hours)\\
Random+cluster & 3 & 120 & 276 & 50 & timeout (24 hours)\\
\hline
\end{tabular}
\end{table}

\section{Conclusion}\label{sec:conclusion}

We described an algorithm to generate difficult model counting instances of small size. Our instances were tested against other benchmarks by state of the art model counters, and proved to be the most challenging of their size. We also empirically proved that for arity $k=3$ the difficulty of both our instances and random instances peaks at a precise location in the constrainedness map. This location is independent of the model counter used and is characterized by a phase transition defined by the probability for the number of models to reach a given threshold. Our experiments then showed that no such characterization can be done for greater arities, because when $k\geq 4$ the peaks of difficulty of different model counters can occur at vastly different constrainedness values. Finally we used our findings to successfully predict a number of clauses where reliably hard instances could be found.

In our experiments, we observed that the threshold associated to the phase transition was around $2^\frac{n}{1.6}$. Future work could focus on determining its exact value. In particular, it would be interesting to find out whether the divisor in the exponent being close to the golden ratio $\frac{1+\sqrt{5}}{2}\approx 1.618$ is significant, or a mere coincidence.

\section*{Acknowledgements}
This material is based upon works supported by the Science Foundation Ireland under Grant No. 12/RC/2289-P2 which is co-funded under the European Regional Development Fund. For the purpose of Open Access, the authors have applied a CC BY public copyright licence to any Author Accepted Manuscript version arising from this submission.

\bibliography{modelcounting}
\bibliographystyle{plain}

\appendix

\section{Tracks 2 and 3 of the 2021 Model Counting Competition}\label{sec:tracks23}

We present track 2 results by solver in Table~\ref{tab:solverst2} and by instance in Figure~\ref{fig:resultst2}. Note that two instances with 6 variables and 4 clauses, both solved by all model counters, are outside the scope of Figure~\ref{fig:resultst2}. For track 3, results by solver can be found in Table~\ref{tab:solverst3} and results by instance in Figure~\ref{fig:resultst3}. Our instances performed decently in track 2, and quite well in track 3, but since the influence of the weight and projection set parameters can be considerable, and we did not include them in our submission, any insight from this data should be taken out with caution.

\begin{table}[t]
\centering
\caption{Results by solver, track 2 instances.}
\label{tab:solverst2}
\begin{tabular}{|c|c|c c c|c|}
\hline
Solver & Ours & Not ours & Small, not ours & Large & Total\\
\hline
SharpSAT-TD & 5/5 & 85/95 & 10/10 & 75/85 & 90/100\\
d4 (cfg 1) & 0/5 & 80/95 & 8/10 & 72/85 & 80/100\\
d4 (cfg 2) & 0/5 & 80/95 & 8/10 & 72/85 & 80/100\\
c2d & 0/5 & 79/95 & 8/10 & 71/85 & 79/100\\
DPMC & 0/5 & 46/95 & 9/10 & 37/85 & 46/100\\
nus-narsimha & 5/5 & 20/95 & 7/10 & 13/85 & 25/100\\
\hline
\end{tabular}
\end{table}

\begin{table}[t]
\centering
\caption{Results by solver, track 3 instances.}
\label{tab:solverst3}
\begin{tabular}{|c|c|c c c|c|}
\hline
Solver & Ours & Not ours & Small, not ours & Large & Total\\
\hline
gpmc & 16/42 & 54/58 & 2/2 & 52/56 & 70/100\\
d4 (cfg 1) & 6/42 & 51/58 & 2/2 & 49/56 & 57/100\\
nus-narasimha (cfg 2) & 9/42 & 43/58 & 1/2 & 42/56 & 52/100\\
nus-narasimha (cfg 1) & 10/42 & 41/58 & 1/2 & 40/56 & 51/100\\
pc2bdd & 2/42 & 39/58 & 0/2 & 39/56 & 41/100\\
ProCount & 4/42 & 17/58 & 2/2 & 15/56 & 21/100\\
c2d & 0/42 & 4/58 & 1/2 & 3/56 & 4/100\\
d4 (cfg 2) & 0/42 & 0/58 & 0/2 & 0/56 & 0/100\\
\hline
\end{tabular}
\end{table}

\begin{figure}
\centering
\begin{subfigure}{\textwidth}
\centering
\includegraphics[width=.80\textwidth]{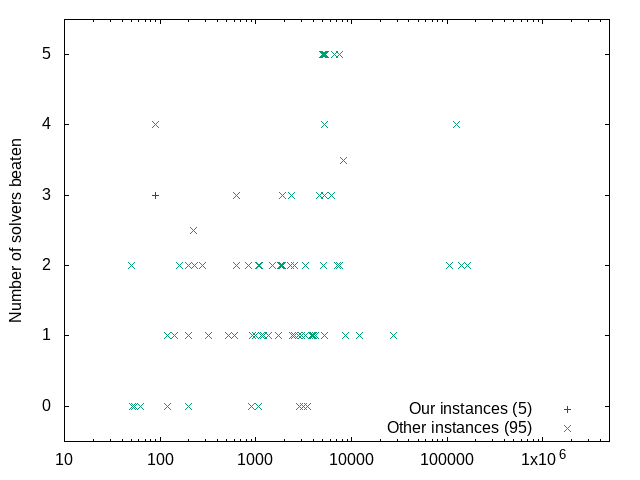}
\caption{Plotted against number of variables.}
\label{fig:resultst2_variables}
\end{subfigure}
\begin{subfigure}{\textwidth}
\centering
\includegraphics[width=.80\textwidth]{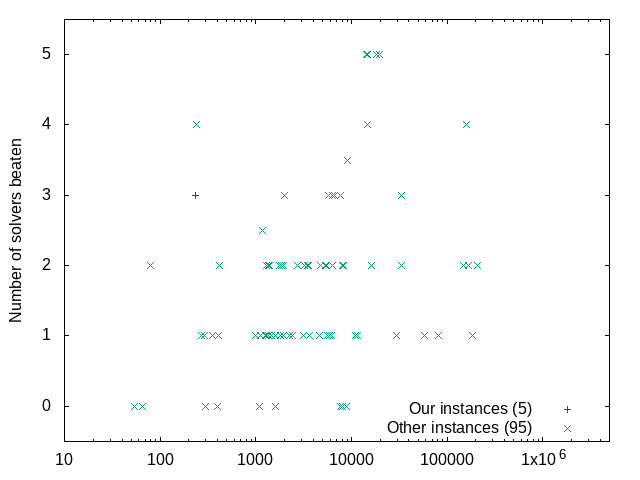}
\caption{Plotted against number of clauses.}
\label{fig:resultst2_clauses}
\end{subfigure}
\caption{Results for track 2 instances.}
\label{fig:resultst2}
\end{figure}

\begin{figure}
\centering
\begin{subfigure}{\textwidth}
\centering
\includegraphics[width=.80\textwidth]{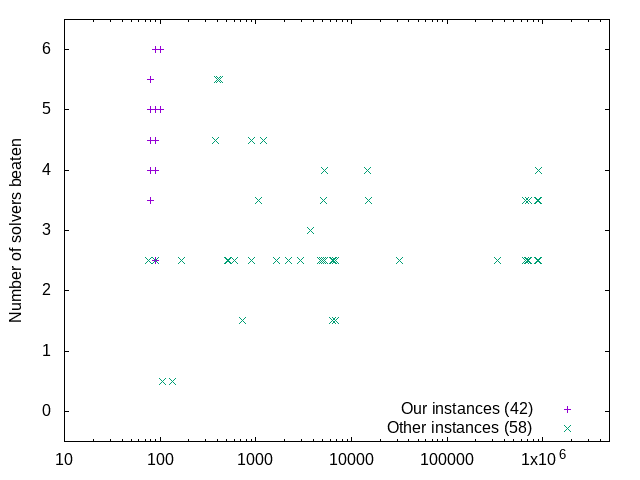}
\caption{Plotted against number of variables.}
\label{fig:resultst3_variables}
\end{subfigure}
\begin{subfigure}{\textwidth}
\centering
\includegraphics[width=.80\textwidth]{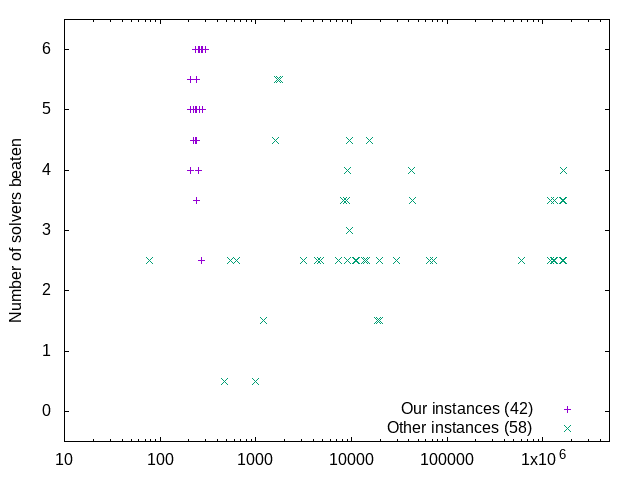}
\caption{Plotted against number of clauses.}
\label{fig:resultst3_clauses}
\end{subfigure}
\caption{Results for track 3 instances.}
\label{fig:resultst3}
\end{figure}

\section{Model counter comparison for random+solution and random+cluster instances}\label{sec:solversG23}

We compare Cachet, {\tt gpmc} and {\tt d4} on random+solution instances in Figures~\ref{fig:solversG2} and on random+cluster instances in Figure~\ref{fig:solversG3}. As is the case for random instances (Figure~\ref{fig:solversG1}), the peaks of difficulty of the model counters drift further apart as the arity increases.

\begin{figure}
\centering
\begin{subfigure}{\textwidth}
\centering
\includegraphics[width=.55\textwidth]{Plots/solvers_G2_K3.png}
\caption{$k=3$, $n=80$.}
\label{fig:solversG2K3}
\end{subfigure}
\begin{subfigure}{\textwidth}
\centering
\includegraphics[width=.55\textwidth]{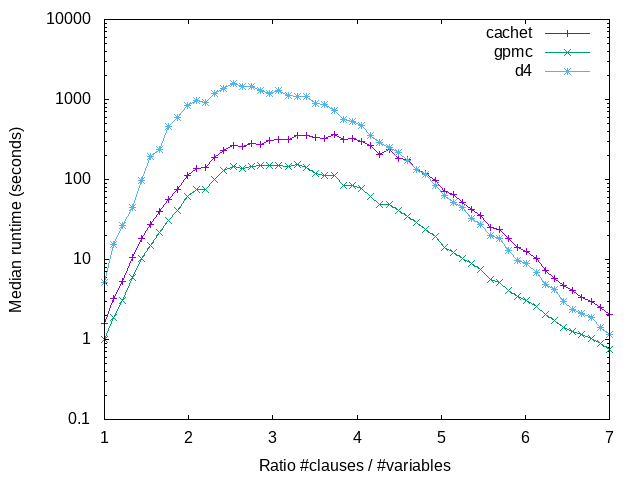}
\caption{$k=4$, $n=55$.}
\label{fig:solversG2K4}
\end{subfigure}
\begin{subfigure}{\textwidth}
\centering
\includegraphics[width=.55\textwidth]{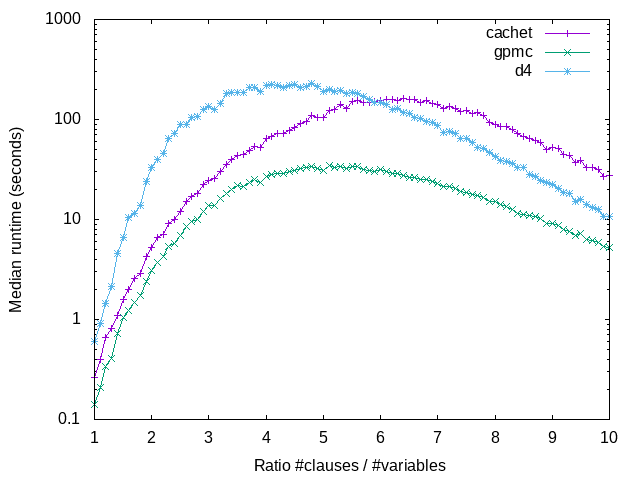}
\caption{$k=5$, $n=40$.}
\label{fig:solversG2K5}
\end{subfigure}
\caption{Model counter comparison for random+solution instances.}
\label{fig:solversG2}
\end{figure}

\begin{figure}
\centering
\begin{subfigure}{\textwidth}
\centering
\includegraphics[width=.55\textwidth]{Plots/solvers_G3_K3.png}
\caption{$k=3$, $n=80$.}
\label{fig:solversG3K3}
\end{subfigure}
\begin{subfigure}{\textwidth}
\centering
\includegraphics[width=.55\textwidth]{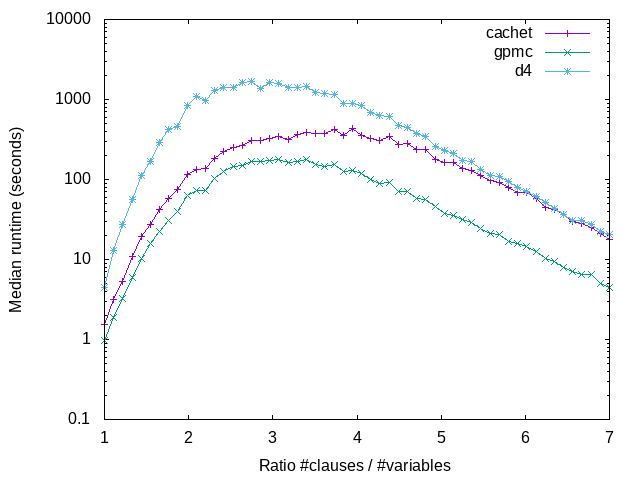}
\caption{$k=4$, $n=55$.}
\label{fig:solversG3K4}
\end{subfigure}
\begin{subfigure}{\textwidth}
\centering
\includegraphics[width=.55\textwidth]{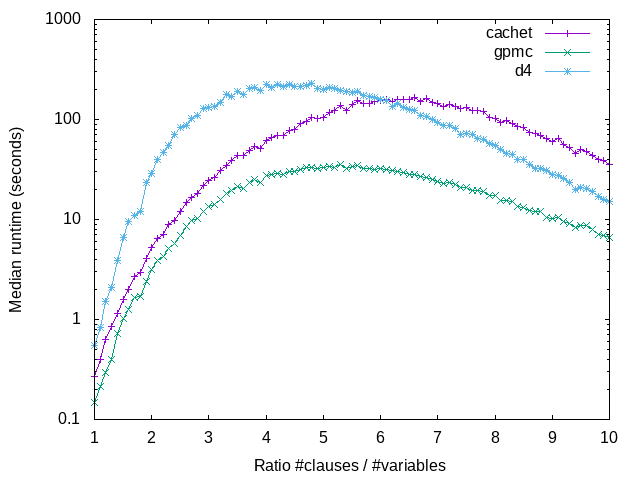}
\caption{$k=5$, $n=40$.}
\label{fig:solversG3K5}
\end{subfigure}
\caption{Model counter comparison for random+cluster instances.}
\label{fig:solversG3}
\end{figure}

\end{document}